\documentclass[twoside]{article}
\usepackage[preprint]{aistats2027}

\usepackage[american]{babel}

\usepackage[round]{natbib}
\renewcommand{\refname}{References}
\usepackage{mathtools}
\usepackage{amsfonts}
\usepackage{amsmath}
\usepackage{amsthm}
\usepackage{amssymb}
\usepackage{thmtools,thm-restate}
\usepackage{nicefrac}

\usepackage{tikz}
\usetikzlibrary{arrows.meta, positioning, calc, decorations.pathmorphing,
                 fit, backgrounds, shapes.geometric, patterns, chains}
\usepackage{graphicx}
\usepackage{subcaption}
\usepackage{float}

\usepackage{booktabs}
\usepackage{array}   
\usepackage{longtable}
\usepackage{multirow}

\usepackage{xcolor}
\usepackage{url}
\usepackage{hyperref}
\hypersetup{hidelinks}
\usepackage{cleveref}
\usepackage[shortlabels]{enumitem}
\usepackage{algorithm}
\usepackage{algorithmic}

\crefname{section}{\S}{\S\S}
\crefname{subsection}{\S}{\S\S}
\crefname{figure}{Fig.}{Figs.}
\crefname{prop}{Prop.}{Props.}
\crefname{theorem}{Thm.}{Thms.}
\crefname{definition}{Defn.}{Defns.}
\crefname{cor}{Corollary}{Corollaries}
\crefname{lem}{Lem.}{Lems.}
\crefname{table}{Tab.}{Tabs.}
\crefname{assum}{Assum.}{Assums.}
\crefname{example}{Ex.}{Exs.}
\crefname{appendix}{Appx.}{Appxs.}

\theoremstyle{plain}

\newtheorem{theorem}{Theorem}
\newtheorem{cor}{Corollary}
\theoremstyle{definition}
\newtheorem{assum}{Assumption}
\newtheorem{definition}{Definition}

\theoremstyle{remark}

\newcommand{\Val}{\mathrm{Val}}
\newcommand{\Solve}{\mathrm{Solve}}
\newcommand{\Proj}{\mathrm{Proj}}
\newcommand{\Pa}{\mathrm{Pa}}
\newcommand{\doop}{\mathrm{do}}
\newcommand{\dom}{\mathrm{dom}}
\newcommand{\id}{\mathrm{id}}
\newcommand{\Diag}{\mathrm{Diag}}

\definecolor{blockblue}{HTML}{0072B2}
\definecolor{blockpink}{HTML}{A64D79}

\begin{document}
\runningtitle{When Do Learned Blocks Preserve Causal Meaning?}
\runningauthor{Zaghw, Brendel, Reizinger}
\twocolumn[
\aistatstitle{When Do Learned Blocks Preserve Causal Meaning?\\ Block Identifiability and Constructive Causal Abstractions}
\aistatsauthor{Karim Zaghw$^{1,2,3}$ \quad Wieland Brendel$^{2,3,\dagger}$ \quad Patrik Reizinger$^{2,3,\dagger}$}
\aistatsaddress{$^1$University of Tübingen \quad $^2$ELLIS Institute Tübingen \\
$^3$Max Planck Institute for Intelligent Systems, Tübingen, Germany \quad $^\dagger$Joint senior authors} ]

\begin{abstract}
Standard identifiability theory assumes one-dimensional latent factors, but learned representations can encode latent information across multidimensional blocks whose internal coordinates are not uniquely determined. When such blocks are used to represent and intervene on higher-level concepts, it is unclear whether their causal meaning survives the transformations left unresolved by block-identifiability. We connect block-identifiability to constructive causal abstraction. We show that block-diagonal invertible transformations followed by compatible coordinate permutations form the maximal class that preserves every constructive causal abstraction built on those blocks. This guarantee assumes that the correspondence between blocks and concepts is updated with the representation. A small synthetic experiment with additive autoencoders shows that high-level interventions can be transferred between independently learned representations while approximately preserving their intended effects.

\end{abstract}

\section{Introduction}\label{sec:intro}

Identifiability theory asks whether a learned representation recovers the latent factors that generated the observations~\citep{hyvarinen_nonlinear_2019,khemakhem_variational_2020}.
Many standard results assume that the learned representation has the same dimensionality as the latent space, with one learned coordinate corresponding to each latent factor.
Under the assumptions of these results, the factors can often be recovered up to a permutation and a separate invertible coordinate-wise transformation.
Because these transformations do not mix the factors, additional assumptions can still allow their causal relations to be recovered~\citep{shimizu_linear_2006,reizinger_jacobian-based_2023}.

In practice, a concept may be encoded across a block of learned coordinates rather than a single coordinate. Several settings motivate this block-level view. Sparse autoencoders (SAEs) can split one concept across multiple learned features~\citep{bricken_towards_2023,mueller_isolation_2025}, while block-sparse featurizers explicitly learn groups of activation directions and recover visual concepts that are typically two- to four-dimensional~\citep{fel_structuring_2026}. Object-centric models represent objects using vector-valued slots~\citep{locatello_object-centric_2020}. Formal block-identifiability results arise in independent subspace analysis (ISA), which recovers independent subspaces up to permutation and invertible within-subspace transformations~\citep{hyvarinen_hoyer_2000,theis_towards_2006}, and in additive decoders, which identify latent blocks up to permutation and block-wise invertible transformations under their assumptions~\citep{lachapelle_additive_2023}. These results recover each block as a whole while leaving its internal coordinates undetermined. Different learned models may therefore represent the same block using different internal coordinates. While practitioners intervene on such blocks by steering feature groups, manipulating object slots, or clamping latent subspaces, block-identifiability does not tell us whether these interventions keep the same high-level causal meaning across different representations. 

We use constructive causal abstractions~\citep{beckers_abstracting_2019,geiger_causal_2023} to state this question precisely. We treat the coordinates of the learned representation as variables in a causal model, which we call the low-level model. The human-interpretable concepts form another causal model, which we call the high-level model. A constructive causal abstraction maps states and interventions in the low-level model to corresponding states and interventions in the high-level model, with each high-level concept computed from one block of learned coordinates. We ask which changes to the learned representation preserve this mapping.

Our contributions are:
\begin{enumerate}[leftmargin=*]
    \item \textbf{We identify the largest class of transformations guaranteed to preserve causal abstractions built from learned blocks.}
    This class consists of invertible transformations within blocks followed by coordinate permutations that carry the block-to-concept correspondence with them.
    It includes the transformations allowed by block-identifiability guarantees for ISA and additive decoders.
    We prove that every transformation in this class preserves every abstraction of this form, while every transformation outside it fails to preserve at least one (\cref{thm:maximality}).

    \item \textbf{We illustrate high-level steering and its transfer between learned representations.}
    In a small synthetic experiment with independently trained additive autoencoders, we translate high-level interventions into low-level block edits.
    We transfer these edits between models using fitted blockwise maps and show that their intended high-level effects are approximately preserved (\cref{sec:experiments}).
\end{enumerate}

\section{Background}\label{sec:background}

\paragraph{Notation.}
Following \citet{geiger_causal_2023}, capital letters denote variables and lowercase letters their values, and boldface denotes a set of variables or a joint setting of one, so that $x\in\Val_X$ and $\mathbf x\in\Val_{\mathbf V_L}$.
Let $\mathbf{S}=\{S_1,\dots,S_m\}$ be the ground-truth causal factors: each $S_i$ is produced by a mechanism $F_{S_i}$ from its parents $\Pa(S_i)$ in the ground-truth directed acyclic graph (DAG) and an exogenous noise variable $N_i$, written $S_i:=F_{S_i}(\Pa(S_i),N_i)$.
The map $f$ takes a factor setting $\mathbf s$ to the activation vector $\mathbf o=f(\mathbf s)$ supplied to the featurizer $g$; this includes generating the network input and passing it through the network. The featurizer returns the learned feature vector, or representation, $\mathbf x=g(\mathbf o)$.
We model the coordinates of $\mathbf x$ as variables of the representation-level SCM $L$. They form the finite set $\mathbf V_L$, so a learned \emph{feature} is a variable $X\in\mathbf V_L$ with value set $\Val_X$, and $\mathbf x\in\Val_{\mathbf V_L}$ is a total setting of $L$.
The high-level SCM $H$ has variables $\mathbf V_H=\{Z^{(1)},\dots,Z^{(B)}\}$, where each $Z^{(b)}$ is a human-interpretable \emph{concept} chosen for intervention and may be computed from a block of learned features rather than from one coordinate.
The ground-truth factors $\mathbf S$ generate the activation vector $\mathbf o$, whereas $\mathbf V_H$ contains the concepts abstracted from the learned representation; we do not assume that they coincide.
For a subset $\mathbf{U}\subseteq\mathbf{V}_L$, write $\Val_{\mathbf{U}}:=\prod_{U\in\mathbf{U}}\Val_{U}$, and the projection $\Proj_{\mathbf{U}}:\Val_{\mathbf{V}_L}\to \Val_{\mathbf{U}}$ restricts a setting to $\mathbf{U}$.
\cref{tab:terminology} in \cref{app:terminology} summarizes this notation. \Cref{fig:pipeline} illustrates this conceptual pipeline.

\begin{figure*}[t]
    \centering
    \includegraphics[width=0.95\textwidth]{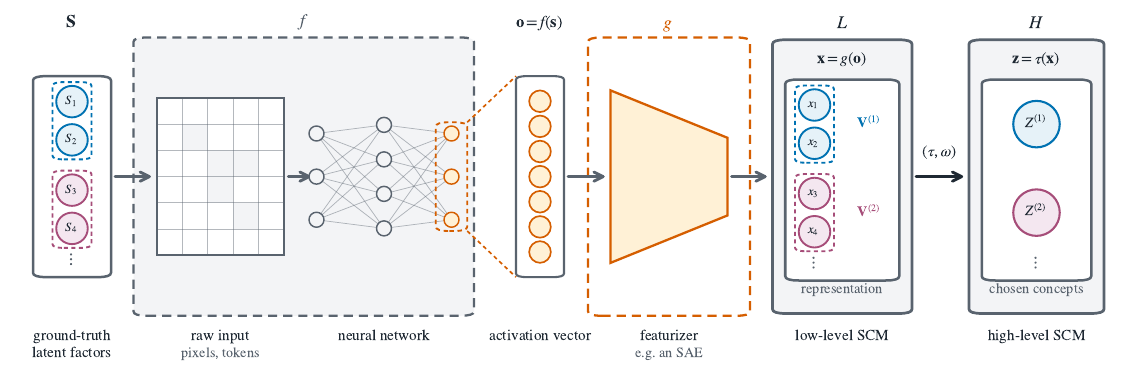}
    \caption{\textbf{From ground-truth factors to a constructive causal abstraction of a learned representation.} The ground-truth causal factors $\mathbf S$ take values $\mathbf s$, which determine the network input and the resulting activation vector $\mathbf o=f(\mathbf s)$. A featurizer $g$, such as a sparse autoencoder (SAE), maps these activations to the learned representation $\mathbf x=g(\mathbf o)$. We treat the coordinates of $\mathbf x$ as variables of the low-level structural causal model (SCM) $L$ and partition them into blocks $\mathbf V^{(b)}$. The constructive causal abstraction $(\tau,\omega)$ maps low-level states and complete-block interventions to states and interventions in the high-level SCM $H$, where each human-interpretable concept $Z^{(b)}$ is computed from its corresponding block. Blue and pink indicate this block-to-concept correspondence. The ground-truth factors $\mathbf S$ generate the activations but need not coincide with the concepts in $H$.}
    \label{fig:pipeline}
\end{figure*}

To compare interventions across levels, we first need a common description of causal models, their solutions, and the interventions applied to them.

\begin{definition}[SCMs and interventions]\label{def:scm}
A \emph{deterministic SCM} is a pair $M=(\Sigma,\{F_X\}_{X\in\mathbf V})$, where $\Sigma=(\mathbf V,\Val)$ is a signature and $F_X:\Val_{\mathbf V}\to\Val_X$ is the mechanism for each $X\in\mathbf V$.
A setting $\mathbf v\in\Val_{\mathbf V}$ is a solution of $M$ if $F_X(\mathbf v)=\Proj_{\{X\}}(\mathbf v)$ for every $X\in\mathbf V$; $\Solve(M)$ denotes the set of all such solutions.
A \emph{hard intervention} on $\mathbf{U}\subseteq\mathbf{V}$ replaces $\{F_X\}_{X\in\mathbf{U}}$ by constants; we identify it with the idempotent map on $\Val_{\mathbf V}$ that overwrites the coordinates in $\mathbf U$, and the empty intervention with the identity map.
A set $\Psi$ of interventionals forms an \emph{intervention algebra} when its composition algebra is isomorphic to that of hard interventions on some signature~\citep{geiger_causal_2023}.
Hard interventions are partially ordered by extension: $I\sqsubseteq I'$ when $I'$ sets every variable $I$ sets, to the same value.
\end{definition}

To say that $L$ and $H$ make consistent causal predictions, mapping an intervention from $L$ to $H$ should give the same high-level outcomes as mapping the resulting low-level solutions to $H$. Exact transformations formalize this requirement~\citep{rubenstein_causal_2017,geiger_causal_2023}.

\begin{definition}[Exact transformation]\label{def:exact}
An \emph{exact transformation} $(\tau,\omega):L\to H$ consists of a surjective state map $\tau:\Val_{\mathbf{V}_L}\to\Val_{\mathbf{V}_H}$ and a surjective order-preserving intervention map $\omega$ such that for all $I\in\dom(\omega)$
\begin{equation}\label{eq:exact}
    \tau\big(\Solve(L^{I})\big)=\Solve\big(H^{\omega(I)}\big)
\end{equation}
\end{definition}

To ask whether a causal abstraction survives a change of representation, we must express the low-level states, mechanisms, and interventions in the new coordinates. Bijective translations formalize this consistent reparameterization~\citep{geiger_causal_2023}.

\begin{definition}[Bijective translation]\label{def:bijective}
Let $M$ be a deterministic SCM, and write $F_{\mathbf V}(\mathbf v):=(F_X(\mathbf v))_{X\in\mathbf V}$ for the map collecting all of its mechanisms. For a bijection $T:\Val_{\mathbf V}\to\Val_{\mathbf V}$, the \emph{bijective translation} $M_T$ has mechanism map
\begin{equation}
    F_{\mathbf V}^T:=T\circ F_{\mathbf V}\circ T^{-1}.
\end{equation}
An intervention $I$ is transported to
\begin{equation}
    I^T:=T\circ I\circ T^{-1}.
\end{equation}
\end{definition}

For any intervention algebra $\Psi$ on $M$, the transported interventions
$\Psi^T:=\{I^T:I\in\Psi\}$ form an intervention algebra on $M_T$, ordered by
$I^T\sqsubseteq J^T$ exactly when $I\sqsubseteq J$.
Every intervention $I$ on $M$ is therefore paired with the transported
intervention $I^T$ on $M_T$. Under this pairing,
\[
    T\bigl(\Solve(M^I)\bigr)=\Solve(M_T^{I^T}),
\]
so the state map $T$ and intervention map $I\mapsto I^T$ form an exact
transformation from $M$ to $M_T$.
Bijective translations compose with other exact transformations~\citep{geiger_causal_2023}.
This composition allows an abstraction from $L$ to $H$ to be carried
over to a reparameterized model $L_T$.

Exactness tells us that the low- and high-level models agree under corresponding interventions, but it does not say where each high-level variable is represented in the low-level model. Our question requires this additional structure because we want each concept to be computed from a particular block of learned coordinates and interventions on that block to correspond to interventions on the concept. Constructive causal abstractions formalize this block-to-concept correspondence through an alignment~\citep{beckers_abstracting_2019}.

\begin{definition}[Alignment and constructive causal abstraction]\label{def:constructive}
An \emph{alignment} between the low-level variables $\mathbf V_L$ and
the high-level variables $\mathbf V_H$ is a pair
$\langle\Pi,\pi\rangle$, where
$\Pi=\{\Pi_Z:Z\in\mathbf V_H\}\cup\{\Pi_\bot\}$ is a partition of
$\mathbf V_L$, and $\pi=\{\pi_Z\}_{Z\in\mathbf V_H}$ is a family of
surjective \emph{cell maps}
$\pi_Z:\Val_{\Pi_Z}\to\Val_Z$.
The cell $\Pi_Z$ contains the low-level variables associated with $Z$,
while the optional cell $\Pi_\bot$ contains variables that are not mapped
to any high-level variable.

The alignment $\langle\Pi,\pi\rangle$ induces the state map
\[
    \tau_\pi(\mathbf x)
    :=\bigl(\pi_Z(\Proj_{\Pi_Z}(\mathbf x))\bigr)_{Z\in\mathbf V_H}
\]
and maps interventions on complete cells according to
\[
    \omega_\pi\bigl(
        \doop(\Pi_Z{=}\mathbf i_Z,\ Z\in A)
    \bigr)
    :=
    \doop\bigl(
        Z{=}\pi_Z(\mathbf i_Z),\ Z\in A
    \bigr),
\]
where $A\subseteq\mathbf V_H$ and
$\mathbf i_Z\in\Val_{\Pi_Z}$.
The high-level model $H$ is a \emph{constructive causal abstraction}
of $L$ under this alignment if $(\tau_\pi,\omega_\pi):L\to H$ is an
exact transformation.
\end{definition}

For a cell map $\pi_Z$, the \emph{fiber} over $z\in\Val_Z$ is $\pi_Z^{-1}(z)$, the set of low-level cell values mapped to $z$. If $\pi_Z$ is many-to-one, values in the same fiber correspond to the same high-level value and therefore cannot be distinguished in $H$.

\section{Preserving Constructive Causal Abstractions}
\label{sec:theory}

We show that constructive causal abstractions are invariant under block-diagonal reparameterizations and that no larger class of reparameterizations preserves all of them, connecting block-identifiability to causal abstraction theory.

\subsection{Setup}\label{sec:setup}

Our aim is to determine whether the ambiguity left by block-identifiability changes the causal meaning assigned to learned blocks. We therefore assume that each learned block serves as the alignment cell for one high-level variable and consider interventions on complete blocks.

\begin{assum}[Block-aligned abstraction]\label{ass:setup}
\leavevmode\par
\begin{enumerate}[label=(\roman*),leftmargin=*]
\item $L=(\Sigma_L,\{F_X\}_{X\in\mathbf V_L})$ is a deterministic SCM with block partition
\begin{equation}\label{eq:blocks}
    \mathbf V_L=\bigsqcup_{b=1}^B\mathbf V^{(b)}.
\end{equation}

\item $(\tau,\omega):L\to H$ is constructive, with each block
$\mathbf V^{(b)}$ serving as the alignment cell for $Z^{(b)}$.
We write
$\pi_b:\Val_{\mathbf V^{(b)}}\to\Val_{Z^{(b)}}$
for the corresponding cell map, for $b=1,\dots,B$.

\item $\Psi$ is the intervention algebra of all hard interventions on zero or more complete blocks, $\dom(\omega)=\Psi$, and $\omega$ is the intervention map induced by the alignment as in \cref{def:constructive}.
\end{enumerate}
\end{assum}

Block-identifiability recovers each block as a whole but leaves its individual coordinates undetermined. We therefore focus on interventions on complete blocks because the corresponding high-level concept can depend on every coordinate in the block. To test whether changes to the internal coordinates affect the block's causal meaning, we collect all invertible within-block coordinate changes in the family $\mathcal D$.

\begin{definition}[Block reparameterizations]\label{def:reparam}
The group of block-diagonal reparameterizations is
\begin{equation}
    \begin{aligned}
        \mathcal D:=\bigl\{\,&\Diag(D^{(1)},\dots,D^{(B)})\ \bigm|\\
        &D^{(b)}:\Val_{\mathbf V^{(b)}}\to\Val_{\mathbf V^{(b)}},\\
        &D^{(b)}\text{ is bijective for }b=1,\dots,B\,\bigr\}.
    \end{aligned}
\end{equation}
\end{definition}

Block-identifiability may also permute whole blocks; we incorporate those permutations below. To characterize the largest class of transformations guaranteed to preserve every abstraction of this form, \cref{thm:maximality} compares these structured transformations with arbitrary bijections of the low-level state space.

To test whether a change of representation preserves the abstraction, we must compare the same high-level intervention before and after the change. Under a bijection $T$, a state $\mathbf y$ in $L_T$ corresponds to $T^{-1}(\mathbf y)$ in $L$, and a low-level intervention $I$ becomes $I^T$. We therefore use $\tau\circ T^{-1}$ for states and define $\omega^T$ so that $I^T$ has the same high-level meaning as $I$.

\begin{definition}[Transported high-level intervention map]\label{def:transport}
For every $I\in\Psi$, define
\begin{equation}
    \omega^T(I^T):=\omega(I).
\end{equation}
\end{definition}

Because $I\mapsto I^T$ is an order-preserving bijection, $\omega^T$ is well-defined, surjective, and order-preserving. By exactness of bijective translations and composition of exact transformations~\citep{geiger_causal_2023}, $(\tau\circ T^{-1},\omega^T):L_T\to H$ is exact. Since every $D\in\mathcal D$ is a bijection, this includes all within-block reparameterizations.

\subsection{The Maximal Class Preserving Constructiveness}
\label{sec:constructive-inv}

The family $\mathcal D$ changes coordinates within blocks, but block-identifiability may also permute whole blocks. More generally, a coordinate permutation can move the coordinates of a block to new positions. To keep each concept associated with the same group of coordinates after this change, we track where each block moves under the permutation.

\begin{definition}[Compatible permutations and transported cells]\label{def:compatible}
A permutation $\sigma$ of the low-level variables $\mathbf V_L$ is
\emph{compatible} if $\Val_{\sigma(X)}=\Val_X$ for every
$X\in\mathbf V_L$. It acts on total settings according to
$\bigl(\sigma(\mathbf x)\bigr)_{\sigma(X)}:=x_X$. For each block
$\mathbf V^{(b)}$, its image $\sigma(\mathbf V^{(b)})$ is the
corresponding \emph{transported cell}, and $\sigma_b$ denotes the
induced bijection from $\Val_{\mathbf V^{(b)}}$ to
$\Val_{\sigma(\mathbf V^{(b)})}$.
\end{definition}

Because $\sigma$ is a permutation and the original blocks partition
$\mathbf V_L$, the transported cells also partition $\mathbf V_L$,
although they may be interleaved with respect to the original partition. We now combine the two changes of representation: $D$ changes coordinates within each block, and $\sigma$ moves those coordinates to their transported cells. The family $\mathcal T$ contains all transformations of this form.

\begin{definition}[Transformation family $\mathcal T$]
\label{def:transformation-class}
The family of block-diagonal reparameterizations followed by compatible
coordinate permutations is
\[
    \mathcal T
    :=
    \{\,\sigma\circ D
    \mid
    \sigma\text{ is compatible and }D\in\mathcal D\,\}.
\]
\end{definition}

Any bijection preserves exactness when the model and its interventions are transported as in \cref{def:bijective}. Our goal is stronger: after changing the representation, each high-level variable should still be computed from one cell, and setting that cell should still be a hard intervention. Both properties hold when $T=\sigma\circ D$. If $I=\doop(\mathbf V^{(b)}{=}\mathbf i_b)$ for some block value $\mathbf i_b\in\Val_{\mathbf V^{(b)}}$, then $I^T$ sets the transported cell $\sigma(\mathbf V^{(b)})$ to the fixed value $\sigma_b(D^{(b)}(\mathbf i_b))$. For a general bijection, either property can fail: the transported intervention may depend on the current state, or a high-level variable may depend on coordinates from several cells.

A two-variable XOR transformation illustrates both failures. Let $\mathbf V=\{X_1,X_2\}$ with binary singleton blocks
$\mathbf V^{(1)}=\{X_1\}$ and $\mathbf V^{(2)}=\{X_2\}$, and let
$\tau(x_1,x_2)=(x_1,x_2)$ map them to
$Z^{(1)}$ and $Z^{(2)}$.
Consider
\[
    \begin{aligned}
        T(x_1,x_2)&=(x_1\oplus x_2,x_2),\\
        T^{-1}(y_1,y_2)&=(y_1\oplus y_2,y_2).
    \end{aligned}
\]
For the hard intervention $I=\doop(X_1{=}0)$, the transported operation is $I^T(y_1,y_2)=(y_2,y_2)$
which copies the current value of $Y_2$ into $Y_1$ rather than setting
$Y_1$ to a constant. It is therefore not a hard intervention in the
new coordinates. The transported state map also satisfies
\[
    (\tau\circ T^{-1})(y_1,y_2)
    =(y_1\oplus y_2,y_2).
\]
Thus $Z^{(1)}$ requires both $Y_1$ and $Y_2$, while $Z^{(2)}$ also
requires $Y_2$, so the two high-level variables cannot be computed from
disjoint cells. Exactness survives relative to the transported
intervention algebra, but constructiveness does not.

The XOR example shows that preserving exactness does not necessarily preserve the block-to-concept structure. We now ask which transformations preserve this structure for every constructive abstraction built on the fixed blocks. For each target SCM $H$, let $\mathcal A(L,H)$ be the set of constructive abstractions $(\tau,\omega):L\to H$ whose cell for $Z^{(b)}$ is $\mathbf V^{(b)}$ for every $b$, with $\dom(\omega)=\Psi$ and $\omega$ the intervention map induced by that alignment (\cref{ass:setup}).

\begin{theorem}[Universal maximality of block-diagonal invariance]\label{thm:maximality}
For the low-level SCM $L$ and block partition in \cref{ass:setup}, fix any bijection $T:\Val_{\mathbf{V}_L}\to\Val_{\mathbf{V}_L}$ and any compatible permutation $\sigma$. The following are equivalent:
\begin{enumerate}[(a)]
    \item For every target SCM $H$ and every $(\tau,\omega)\in\mathcal A(L,H)$, the transported pair $(\tau\circ T^{-1},\omega^T):L_T\to H$ is constructive with cells $\sigma(\mathbf V^{(b)})$ for every $b$.
    \item $T=\sigma\circ D$ for some $D\in\mathcal D$.
\end{enumerate}
\end{theorem}

The proof is given in \cref{app:proof-maximality}. The theorem identifies the largest family of transformations for which constructiveness is guaranteed, connecting the result directly to block-identifiability. In block-identifiability results, the permitted permutations are whole-block permutations, meaning that each original block is mapped onto another original block. Together with invertible transformations within blocks, these transformations are a special case of $\mathcal T$. The family $\mathcal T$ is broader because it also permits compatible coordinate permutations that move a block to an interleaved transported cell.

There is one important qualification: the guarantee is universal. Every transformation in $\mathcal T$ preserves every constructive abstraction built on the fixed blocks. A transformation outside $\mathcal T$ must break at least one such abstraction, but it may still preserve a particular one.
For example, take four binary variables with blocks $\textcolor{blockblue}{\{X_1,X_2\}}$ and $\textcolor{blockpink}{\{X_3,X_4\}}$, and let $\tau(\mathbf x)=(\textcolor{blockblue}{x_1\oplus x_2},\,\textcolor{blockpink}{x_3\oplus x_4})$. Consider the transformation
\[
    T_+(\mathbf x)=
    \bigl(
        \textcolor{blockblue}{x_1}\oplus\textcolor{blockpink}{x_3},\,
        \textcolor{blockblue}{x_2}\oplus\textcolor{blockpink}{x_3},\,
        \textcolor{blockpink}{x_3},\,
        \textcolor{blockpink}{x_4}
    \bigr),
\]
which mixes the two blocks. Nevertheless,
\[
    (\tau\circ T_+^{-1})(\mathbf y)
    =
    \bigl(
        \textcolor{blockblue}{y_1\oplus y_2},\,
        \textcolor{blockpink}{y_3\oplus y_4}
    \bigr).
\]
The first high-level variable is unchanged because the added value $\textcolor{blockpink}{x_3}$ appears twice and therefore disappears under XOR: $(\textcolor{blockblue}{x_1}\oplus\textcolor{blockpink}{x_3})\oplus(\textcolor{blockblue}{x_2}\oplus\textcolor{blockpink}{x_3})=\textcolor{blockblue}{x_1\oplus x_2}$. Thus $T_+\notin\mathcal T$ still preserves this particular abstraction.

This does not contradict \cref{thm:maximality}, which concerns preservation of every abstraction built on the fixed blocks. Its proof shows that a transformation outside $\mathcal T$ fails for the identity block abstraction, whose cell maps retain the complete value of each block. Preserving one particular abstraction therefore does not imply that $T\in\mathcal T$.


\begin{figure*}[t]
    \centering
    \includegraphics[width=.85\textwidth]{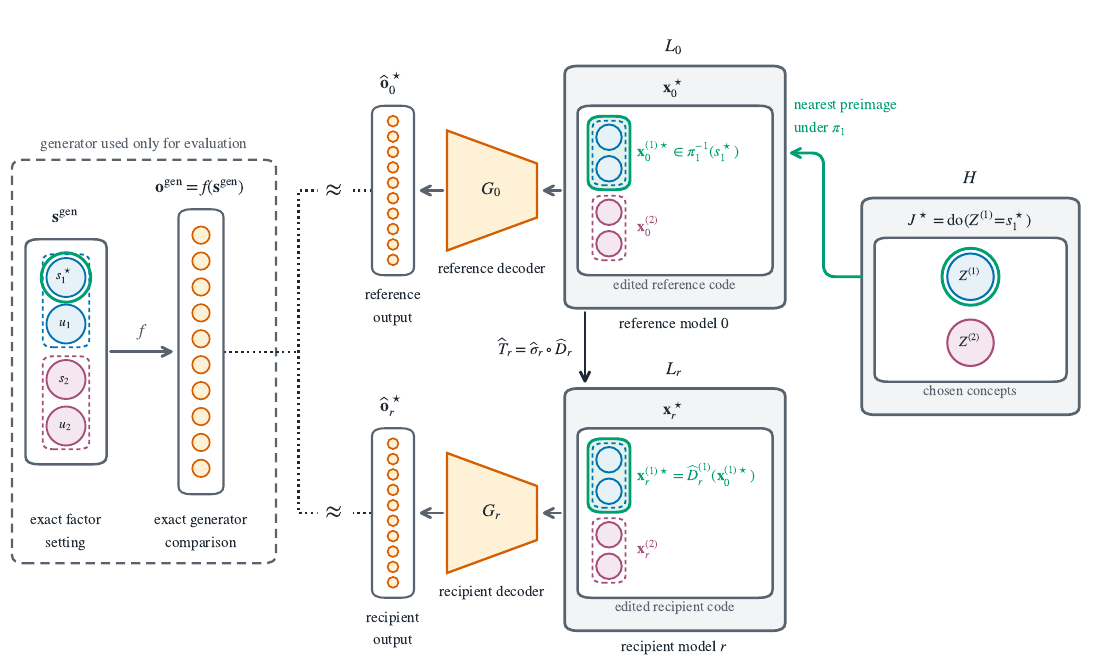}
    \caption{\textbf{Concept-directed steering and transport between independently trained additive autoencoders.} Starting from a base observation, the high-level intervention $J^\star=\doop(Z^{(1)}{=}s_1^\star)$ in $H$ is implemented in the reference model $L_0$ by editing block~1. Here, $\pi_1$ maps that block to $Z^{(1)}$, and $\mathbf x_0^{(1)\star}$ is the closest block value satisfying $\pi_1(\mathbf x_0^{(1)\star})=s_1^\star$; block~2 remains unchanged. The fitted transformation $\widehat T_r=\widehat\sigma_r\circ\widehat D_r$ transfers the edit to recipient model $L_r$, where $\widehat D_r$ maps corresponding blocks and $\widehat\sigma_r$ accounts for their ordering. The decoders $G_0$ and $G_r$ produce the edited outputs $\widehat{\mathbf o}_0^\star$ and $\widehat{\mathbf o}_r^\star$. For evaluation, the exact factor setting $\mathbf s^{\mathrm{gen}}=(s_1^\star,u_1,s_2,u_2)$ changes only $s_1$, setting it to the intervention value $s_1^\star$, and the data-generating map $f$ produces the comparison observation $\mathbf o^{\mathrm{gen}}=f(\mathbf s^{\mathrm{gen}})$. Both decoded outputs are compared with this same observation, as indicated by the dotted lines. Green marks the intervention, while blue and pink identify the two blocks throughout.}
    \label{fig:experiment-transport}
\end{figure*}

\paragraph{Implication for block-identifiable representations.}
Under \cref{ass:setup}, the transformations left unresolved by block-identifiability do not change the predicted high-level effect of an intervention. After transporting the alignment and interventions, whole-block permutations and invertible within-block reparameterizations yield the same solution set in $H$ for every corresponding intervention.
\Cref{thm:maximality} identifies the boundary of this universal guarantee: every $T\in\mathcal T$ preserves every block-aligned constructive abstraction, whereas each $T\notin\mathcal T$ breaks at least one, although a particular abstraction may survive.
Thus complete blocks, rather than their non-identifiable internal coordinates, are the stable units of intervention. When $\pi_b$ is many-to-one, the high-level intervention $\doop(Z^{(b)}{=}z)$ corresponds to every complete-block intervention $\doop(\mathbf V^{(b)}{=}\mathbf i_b)$ with $\mathbf i_b\in\pi_b^{-1}(z)$. Two abstractions whose cell maps have the same fibers in every cell differ only by a relabeling of their high-level values (\cref{cor:unique}). The abstraction does not distinguish these low-level interventions, so implementing a high-level intervention requires an additional selection rule. This ambiguity becomes relevant in \cref{sec:experiments}, where a nearest-point rule selects one representative from the fiber.

\section{Experiments
}
\label{sec:experiments}

\paragraph{Goal.}
We test whether a high-level intervention in the human-interpretable model $H$ can be translated into a block intervention in a reference low-level model $L_0$ and transported to independently trained low-level models $L_r$ while preserving its high-level effect. \Cref{fig:experiment-transport} summarizes this steering, transport, and evaluation procedure.

\paragraph{Setting.}
Synthetic factors $(s_1,u_1,s_2,u_2)$ generate observations through the nonlinear additive map
$\mathbf o=A_1\phi(s_1,u_1)+A_2\phi(s_2,u_2)$, where $\phi(s,u)=(s,u,s^2,su,u^2)$; the distribution and mixing matrix are specified in \cref{app:exp-additive}.
Following additive-decoder models~\citep{lachapelle_additive_2023}, model $r$ has an encoder $E_r(\mathbf o)=\mathbf x_r$ and a decoder $G_r(\mathbf x_r)=\widehat{\mathbf o}_r$. The four-dimensional representation $\mathbf x_r$ has two two-dimensional blocks $\mathbf x_r^{(1)},\mathbf x_r^{(2)}$. Each block is passed through its own nonlinear decoder branch, and the two outputs are added:
\[
    \widehat{\mathbf o}_r=G_r^{(1)}(\mathbf x_r^{(1)})+G_r^{(2)}(\mathbf x_r^{(2)}).
\]
We treat $\mathbf x_r$ as the state of the low-level SCM $L_r$ and the two learned blocks as its cells $\mathbf V^{(1)},\mathbf V^{(2)}$.
The high-level SCM $H$ keeps only $Z^{(1)}=s_1$ and $Z^{(2)}=s_2$; omitting $u_1,u_2$ makes each cell map $\pi_b$ many-to-one.

\paragraph{Protocol.}
We train one reference autoencoder $(E_0,G_0)$ and ten additional autoencoders $(E_r,G_r)$ that may receive an intervention transported from the reference model. After training, we freeze every model and retain for transport only those that reconstruct held-out observations with normalized mean squared error (NMSE) at most $0.01$. The architecture, training settings, data splits, and definition of NMSE are given in \cref{app:exp-additive}, and we provide the code as supplementary material.

We do not initially know which learned block corresponds to $Z^{(1)}$ and which corresponds to $Z^{(2)}$.
We therefore use fresh examples in which each observation $\mathbf o$ is paired with its known concept values $(Z^{(1)},Z^{(2)})=(s_1,s_2)$, and encode $\mathbf o$ with $E_0$.
Let $\pi_{bj}$ denote the affine map fitted from learned block $\mathbf V^{(b)}$ to concept $Z^{(j)}$.
We fit all four maps and compare the two possible one-to-one assignments: $(\pi_{11},\pi_{22})$ and $(\pi_{12},\pi_{21})$.
On held-out examples, we choose the assignment whose two maps best predict $Z^{(1)}$ and $Z^{(2)}$.
After indexing the blocks by their selected concepts, these maps become the cell maps $\pi_1$ and $\pi_2$, defining the empirical reference abstraction
\begin{equation}
    \tau(\mathbf x_0):=\bigl(\pi_1(\mathbf x_0^{(1)}),\pi_2(\mathbf x_0^{(2)})\bigr),
    \label{eq:experiment-tau}
\end{equation}
which approximates $(s_1,s_2)$. The nuisance variables $u_1,u_2$ are not used to fit these maps.

We estimate $T_r=\sigma_r\circ D_r$ by encoding the same observations with $E_0$ and $E_r$, giving paired codes $\mathbf x_0=E_0(\mathbf o)$ and $\mathbf x_r=E_r(\mathbf o)$. We fit affine maps between all pairs of blocks and choose the one-to-one pairing with the lowest held-out code error. This gives both the estimated permutation $\widehat\sigma_r$ and the fitted within-block maps $\widehat D_r^{(b)}$. We then index recipient blocks by their matched reference blocks, so that $\widehat D_r^{(b)}(\mathbf x_0^{(b)})\approx\mathbf x_r^{(b)}$. In the original output order, $\widehat T_r=\widehat\sigma_r\circ\Diag(\widehat D_r^{(1)},\widehat D_r^{(2)})$.

The purpose of $H$ is to express interventions in the concepts we understand, but the edit must ultimately be applied to $L$. The high-level intervention $\doop(Z^{(b)}{=}s_b^\star)$ does not have a unique low-level implementation because $\pi_b$ is many-to-one. Given a base observation, we encode $\mathbf x_0=E_0(\mathbf o_{\mathrm{base}})$ and keep its other block fixed while choosing $\mathbf x_0^{(b)\star}\in\pi_b^{-1}(s_b^\star)$ closest to the current block $\mathbf x_0^{(b)}$. Thus $H$ specifies the fiber, while the nearest-point rule uses information unavailable in $H$ to select one low-level representative. Replacing block $b$ gives the edited reference code $\mathbf x_0^\star$, which decodes as $\widehat{\mathbf o}_0^\star=G_0(\mathbf x_0^\star)$. We encode the same base observation as $\mathbf x_r=E_r(\mathbf o_{\mathrm{base}})$ and replace its corresponding block by
\begin{equation}
    \mathbf x_r^{(b)\star}:=\widehat D_r^{(b)}(\mathbf x_0^{(b)\star}),
    \label{eq:transported-steering-block}
\end{equation}
and decode the edited recipient code as $\widehat{\mathbf o}_r^\star=G_r(\mathbf x_r^\star)$. The closed-form affine projection used to choose $\mathbf x_0^{(b)\star}$ is given in \cref{app:exp-additive}.

The known map used to generate the synthetic observations is used only to evaluate the edits. For the intervention $\doop(Z^{(b)}{=}s_b^\star)$, we construct
\[
    \mathbf o^{\mathrm{gen}}
    =A_b\phi(s_b^\star,u_b)
     +A_{3-b}\phi(s_{3-b},u_{3-b}),
\]
where $u_b$ and $(s_{3-b},u_{3-b})$ come from the base observation. Thus only $s_b$ is changed, setting it to $s_b^\star$. We compare both the reference output $\widehat{\mathbf o}_0^\star$ and every recipient output $\widehat{\mathbf o}_r^\star$ with this same observation. We measure agreement using the NMSE between each decoded output and $\mathbf o^{\mathrm{gen}}$. Because the generator is injective, a low NMSE indicates that the edit sets $s_b$ to $s_b^\star$ while preserving its nuisance value and the other factor pair. These factor values and the generator are used only for evaluation, not to choose or transport the edit. Full experimental details and additional diagnostics are in \cref{app:exp-additive}.

\paragraph{Results.}
The reference model and four of the ten additional autoencoders pass the reconstruction threshold; the other six fail clearly and are excluded before steering. The reference cell maps identify their concepts unambiguously, and paired codes identify the corresponding blocks in every retained recipient model. Steering succeeds in the reference model and after transport to all four recipient models: every decoded observation has NMSE at most $0.012$ against the common exact generator comparison for the corresponding intervention. \Cref{tab:steering-results} reports the reconstruction and steering NMSE for every retained model.

\begin{table}[t]
\centering
\caption{\textbf{Reconstruction and concept-intervention accuracy across retained additive autoencoders.} Rows show one reference additive autoencoder and four independently trained recipient additive autoencoders, identified by training seed. ``Recon.'' is the normalized mean squared error (NMSE) between each model's decoded reconstructions and held-out synthetic input observations; only models with reconstruction NMSE at most $0.01$ are retained. The $Z^{(b)}$ columns report intervention error. An intervention on $Z^{(b)}$ is implemented as a block edit in the reference model and transported to each recipient model. Each decoded edited output is compared with the exact synthetic observation obtained by changing only $s_b$ to the intervention value while keeping $u_b$ and $(s_{3-b},u_{3-b})$ fixed. Each NMSE pools $5{,}000$ examples; lower values are better.}
\label{tab:steering-results}
\small
\begin{tabular}{@{}lccc@{}}
\toprule
Model (seed) & Recon. & $Z^{(1)}$ & $Z^{(2)}$ \\
\midrule
Reference (0)  & $0.007312$ & $0.007407$ & $0.007378$ \\
Recipient (2)  & $0.007927$ & $0.011931$ & $0.010487$ \\
Recipient (3)  & $0.007762$ & $0.009342$ & $0.010307$ \\
Recipient (5)  & $0.007825$ & $0.010572$ & $0.010985$ \\
Recipient (10) & $0.007779$ & $0.009505$ & $0.010923$ \\
\bottomrule
\end{tabular}
\end{table}

This experiment is an approximate illustration of the theory; it does not establish that the learned models form an exact constructive abstraction. The accompanying supplementary code generates the synthetic data, trains the models, and reproduces the reported results from fixed seeds.

\section{Related Works}\label{sec:related}

\paragraph{Causal abstractions and alignment.}
Prior work introduced exact transformations and constructive causal abstractions for relating causal models at different levels~\citep{rubenstein_causal_2017,beckers_abstracting_2019,geiger_causal_2023} and applied causal abstractions to neural networks~\citep{geiger_causal_2021}.
Distributed alignment search (DAS) finds a rotation of a neural representation that places each specified high-level variable in its own subspace, allowing interventions on that variable to be implemented by editing the corresponding subspace~\citep{geiger_finding_2024}.
DAS searches for an alignment to one high-level model; we instead characterize the representation changes guaranteed to preserve any constructive abstraction with the given blocks.
Other work studies how abstractions can be discovered.
Causal feature learning~\citep{chalupka_causal_2017} and unsupervised causal abstraction discovery~\citep{saulus_unsupervised_2026} seek macro-variables or high-level models, while \citet{li_identifiability_2025} use paired observations before and after unknown interventions to determine which latent variables cannot be distinguished and must therefore be represented as one higher-level variable.
Our theory assumes that the low- and high-level models, blocks, and abstraction are given and asks which changes of representation preserve that abstraction.

\paragraph{Block identifiability.}
Independent subspace analysis (ISA) identifies latent subspaces up to their permutation and invertible transformations within them~\citep{hyvarinen_hoyer_2000,theis_towards_2006}; additive decoders give an analogous guarantee under additive structure~\citep{lachapelle_additive_2023}.
These transformations are a special case of the family $\mathcal T$ characterized by \cref{thm:maximality}.
Causal representation learning more broadly seeks to recover latent causal variables or their structure~\citep{scholkopf_towards_2021}.
Our theory starts from a given abstraction and studies whether it remains valid under the transformations left open by block identifiability.

\paragraph{Feature blocks and steering.}
For an SAE, each output coordinate of the encoder is one learned feature in our terminology. SAEs with different numbers of output coordinates need not learn the same features: a concept represented by one coordinate in an SAE with fewer output coordinates may be represented by several more specific coordinates in an SAE with more output coordinates~\citep{bricken_towards_2023,leask_sparse_2024}.
Sparse Shift Autoencoders identify one-dimensional steering directions up to permutation and rescaling under their assumptions~\citep{joshi_identifiable_2025}.
Block-sparse featurizers instead represent a concept using a multidimensional block and recover low-dimensional visual concept manifolds whose internal coordinates are defined only up to rotation~\citep{fel_structuring_2026}.
Together, these works motivate asking when steering a concept represented by a block retains its high-level causal meaning under changes of representation.

\paragraph{Approximate causal abstractions.}
Approximate causal abstraction allows the high-level model to reproduce the low-level system only up to a quantified error~\citep{beckers_approximate_2019}, and abstraction maps can be learned by enforcing consistency across interventional distributions~\citep{zennaro_abstraction_2023}.
Our maximality theorem concerns exact abstractions; the synthetic experiment measures approximate steering and transport rather than proving an approximate abstraction guarantee.

\section{Discussion}\label{sec:discussion}

\paragraph{Limitations.}
Our result assumes deterministic SCMs $L$ and $H$, a fixed block partition and its correspondence to the high-level variables, and a given exact constructive abstraction $(\tau,\omega)$.
It applies to hard interventions on complete blocks and does not show how these assumptions can be learned from data.
Our synthetic experiment studies a controlled and approximate setting: the additive architecture matches the synthetic generator, concept-labelled observations are used to fit the cell maps $\pi_b$, observations encoded by different additive autoencoders are used to fit transformations between their representations, and only autoencoders that reconstruct accurately are retained.
The experiment therefore illustrates concept-directed steering and its transport in a simplified setting, but it does not establish an exact abstraction or provide a general method for discovering one.

\paragraph{Future work.}
The main open problems are to recover the blocks and $(\tau,\omega)$ under explicit identifiability assumptions; extend the maximality result to stochastic SCMs, soft interventions, and quantitatively approximate abstractions; and evaluate concept-directed steering and transport in real neural representations, where exact generator comparisons are unavailable.

\paragraph{Conclusion.}
Under a given block-aligned constructive abstraction, the ambiguity left by block-identifiability does not destroy the high-level causal interpretation.
Block-diagonal reparameterizations, up to compatible coordinate permutations, are the maximal class guaranteed to preserve every such abstraction under the transported alignment; transformations outside this class may preserve a particular abstraction but have no universal guarantee (\cref{thm:maximality}).
In a controlled synthetic setting, we illustrate this connection by translating high-level concept interventions into edits of entire learned blocks and transporting those edits between independently trained representations (\cref{sec:experiments}).

\subsection*{AI use statement}
Generative AI tools were used to assist with implementing, debugging, and testing the experimental code and to suggest limited edits to the manuscript for clarity and concision. All reported results were produced by executing the code, and the authors reviewed the AI-assisted code, experimental outputs, and manuscript edits. The authors take responsibility for the final content of this work, including all text, claims, and artifacts produced with the aid of generative AI.

\subsection*{Acknowledgments}
The authors thank Shruti Joshi and Théo Saulus for their useful comments and suggestions regarding causal abstractions and practical identifiability. P.R. acknowledges his membership in the European Laboratory for Learning and Intelligent Systems (ELLIS) PhD program and thanks the International Max Planck Research School for Intelligent Systems (IMPRS-IS) for its support. This work was supported by the German Federal Ministry of Education and Research (BMBF) through the T\"ubingen AI Center, FKZ: 01IS18039A, and by a grant from Coefficient Giving to Aaron Mueller. W.B. acknowledges financial support from an Emmy Noether Grant funded by the German Research Foundation (DFG) under grant no.\ BR 6382/1-1, and from the Open Philanthropy Foundation, funded by the Good Ventures Foundation. W.B. is a member of the Machine Learning Cluster of Excellence, EXC number 2064/1 -- Project number 390727645. This research also used compute resources at the T\"ubingen Machine Learning Cloud, DFG FKZ INST 37/1057-1 FUGG. Finally, we thank the organisers and participants of the Fourth Bellairs Workshop on Causality, held at the McGill University Bellairs Research Institute from 14--21 February 2025, for sharing their knowledge about causal abstractions and mechanistic interpretability, in particular Atticus Geiger and Thomas Icard.

\bibliography{references}

\clearpage
\appendix
\thispagestyle{empty}
\onecolumn
\aistatstitle{When Do Learned Blocks Preserve Causal Meaning?\\ Block Identifiability and Constructive Causal Abstractions: Supplementary Materials}

\section{Terminology}\label{app:terminology}

\cref{tab:terminology} collects the terms the paper builds on with their notation, a one-line reading, and the place each is defined.

\begingroup
\footnotesize
\begin{longtable}{@{}>{\raggedright\arraybackslash}p{0.19\textwidth}>{\raggedright\arraybackslash}p{0.12\textwidth}>{\raggedright\arraybackslash}p{0.39\textwidth}>{\raggedright\arraybackslash}p{0.17\textwidth}@{}}
\caption{Core terminology and notation used in the paper.}\label{tab:terminology}\\
\toprule
Term & Notation & Meaning & Defined in \\
\midrule
\endfirsthead
\toprule
Term & Notation & Meaning & Defined in \\
\midrule
\endhead
Representation (feature vector) & $\mathbf{x}=g(\mathbf{o})$ & Featurizer output whose coordinates form $\mathbf V_L$ & \cref{sec:background} \\
Activation vector & $\mathbf o$ & Activation read from the network and supplied to the featurizer $g$ & \cref{sec:background} \\
Featurizer & $g$ & Maps network activations to $\mathbf x$; examples include SAEs and block-sparse featurizers & \cref{sec:background} \\
Learned feature & $X\in\mathbf V_L$ & Coordinate variable of the featurizer output & \cref{sec:background} \\
Concept & $Z^{(b)}\in\mathbf V_H$ & High-level variable, possibly encoded across several learned coordinates & \cref{sec:background} \\
Low-level model & $L$, $\mathbf V_L$ & Representation-level SCM on the featurizer output, whose coordinates carry the block partition & \cref{sec:background,ass:setup} \\
High-level model & $H$, $\mathbf V_H$ & SCM over the human-interpretable concepts & \cref{sec:background} \\
Deterministic SCM & $M$ & Signature $\Sigma$ plus mechanisms $\{F_X\}$; $\Solve(M)$ are its fixed points & \cref{sec:background}; \citet{geiger_causal_2023} \\
Hard intervention & $I$ (on $L$), $J$ (on $H$) & Replaces the mechanisms of a variable set by constants; identified with its overwrite map; the empty intervention is the identity & \cref{sec:background} \\
Intervention algebra & $\Psi$ & Composition algebra isomorphic to hard interventions; here, hard interventions on zero or more complete blocks & \cref{sec:background}; \cref{ass:setup} \\
Exact transformation & $(\tau,\omega)$ & State and intervention maps commuting with $\Solve$ & \cref{eq:exact}; \citet{rubenstein_causal_2017,geiger_causal_2023} \\
Bijective translation & $T:L\to L_T$ & Coordinate change with transported interventions $I^T=T\circ I\circ T^{-1}$; exact relative to the transported intervention algebra & \cref{def:bijective}; \citet{geiger_causal_2023} \\
Constructive abstraction & --- & Exact transformation whose state and intervention maps are induced by an alignment & \cref{def:constructive}; \citet{beckers_abstracting_2019,geiger_causal_2023} \\
Block partition & $\mathbf{V}^{(b)}$ & Disjoint grouping of the low-level variables into $B$ blocks & \cref{eq:blocks} \\
Alignment cell, cell map & $\mathbf V^{(b)}$, $\pi_b$ & The block from which $Z^{(b)}$ is computed and the surjection producing its value & \cref{ass:setup} \\
Block-diagonal reparameterization & $D\in\mathcal D$ & Bijection acting within each block, mixing none & \cref{def:reparam} \\
Transformation family & $T\in\mathcal T$ & Maps $T=\sigma\circ D$ with compatible $\sigma$ and $D\in\mathcal D$ & \cref{def:transformation-class,thm:maximality} \\
Transported model & $L_T$ & Mechanisms conjugated through $T$, so $T:L\to L_T$ is a bijective translation & \cref{def:bijective} \\
Compatible permutation, transported cells & $\sigma$, $\sigma(\mathbf V^{(b)})$ & Permutation of coordinates with matching value domains, carrying each block to a possibly interleaved cell; whole-block permutations are the special case & \cref{def:compatible} \\
Fibers & $\pi_b^{-1}(z)$ & The low-level cell values mapped to the same high-level value; equal fibers make two abstractions differ by a relabeling & \cref{sec:constructive-inv,cor:unique} \\
Block-identifiability & --- & Recovery of blocks up to their permutation and invertible within-block maps; ISA gives the linear within-subspace case & \cref{sec:intro} \\
\bottomrule
\end{longtable}
\endgroup

\section{Proofs}\label{app:proofs}

This section contains the proof of the main theorem and a secondary corollary.

\subsection{Universal maximality of block-diagonal invariance}\label{app:proof-maximality}

\begin{proof}[Proof of \cref{thm:maximality}]
\emph{(b)$\Rightarrow$(a).}
Assume $T=\sigma\circ D$ with $D=\Diag(D^{(1)},\ldots,D^{(B)})$, and fix a target $H$ and $(\tau,\omega)\in\mathcal A(L,H)$.
Constructiveness of the transported pair requires two properties: an alignment with cells $\sigma(\mathbf V^{(b)})$ inducing $\tau\circ T^{-1}$, and exactness of $(\tau\circ T^{-1},\omega^T):L_T\to H$.

\smallskip
\noindent\emph{Step 1: the transported alignment.}
Let $\pi_b:\Val_{\mathbf V^{(b)}}\twoheadrightarrow\Val_{Z^{(b)}}$ be the original cell maps, so
\begin{equation}
\tau(\mathbf x)=\bigl(\pi_b(\mathbf x_{\mathbf V^{(b)}})\bigr)_{b=1}^B.
\end{equation}
For $\mathbf y=T(\mathbf x)$, $\sigma^{-1}$ moves the value in $\sigma(\mathbf V^{(b)})$ back to $\mathbf V^{(b)}$, and $(D^{(b)})^{-1}$ undoes the within-block change of coordinates:
\begin{equation}
(T^{-1}(\mathbf y))_{\mathbf V^{(b)}}=(D^{(b)})^{-1}\!\left(\sigma_b^{-1}(\mathbf y_{\sigma(\mathbf V^{(b)})})\right).
\end{equation}
Define the cell map on the transported cell $\sigma(\mathbf V^{(b)})$ by
\begin{equation}
\pi_b^T:=\pi_b\circ(D^{(b)})^{-1}\circ\sigma_b^{-1}.
\end{equation}
The cells $\sigma(\mathbf V^{(b)})$ partition $\mathbf V_L$ (\cref{def:compatible}); each $\pi_b^T$ is surjective because $\pi_b$ is and the remaining factors are bijective; and substituting the block formula for $T^{-1}$ into $\tau$ gives
\begin{equation}
(\tau\circ T^{-1})(\mathbf y)=\bigl(\pi_b^T(\mathbf y_{\sigma(\mathbf V^{(b)})})\bigr)_{b=1}^B,
\end{equation}
the required alignment-induced form.

\smallskip
\noindent\emph{Step 2: exactness.}
Let $I\in\dom(\omega)=\Psi$ set the complete blocks indexed by $S\subseteq\{1,\ldots,B\}$ to values $\mathbf i_b$.
Its conjugate $I^T$ sets each transported cell $\sigma(\mathbf V^{(b)})$, $b\in S$, to
\begin{equation}
\mathbf i_b^T:=\sigma_b(D^{(b)}(\mathbf i_b)),
\end{equation}
and the transported cell map assigns this cell value the same macro-value:
\begin{equation}
\pi_b^T(\mathbf i_b^T)=\pi_b(\mathbf i_b).
\end{equation}
Hence $\omega^T$, as defined in \cref{def:transport}, coincides with the intervention map induced by the transported alignment (\cref{def:constructive}).
Since $T:L\to L_T$ is a bijective translation (\cref{def:bijective}), and bijective translations are exact~\citep{geiger_causal_2023},
\begin{equation}
\Solve(L_T^{I^T})=T(\Solve(L^I)),
\end{equation}
so
\begin{align}
(\tau\circ T^{-1})(\Solve(L_T^{I^T}))
&=\tau(\Solve(L^I))\\
&=\Solve(H^{\omega(I)})\\
&=\Solve(H^{\omega^T(I^T)}),
\end{align}
using the transported-solution equation, exactness of $(\tau,\omega)$, and $\omega^T(I^T)=\omega(I)$.
The transported pair is exact and alignment-induced, hence constructive, proving (a).

\emph{(a)$\Rightarrow$(b).}
Fix the compatible permutation $\sigma$ from the theorem statement.
We construct a block-diagonal bijection $D$ with $T=\sigma\circ D$ via the identity block abstraction: a coarse abstraction may discard the information needed to recover $T$, whereas the identity abstraction retains the complete value of every block.

\smallskip
\noindent\emph{Step 1: the identity abstraction lies in the quantified family.}
Define
\begin{equation}
\tau_{\mathrm{id}}(\mathbf x)=\bigl(\mathbf x_{\mathbf V^{(b)}}\bigr)_{b=1}^B,
\end{equation}
which packages each complete block as one product-valued macro-variable and loses no information.
Let $H_{\mathrm{id}}$ be the SCM with mechanisms conjugated through $\tau_{\mathrm{id}}$ (\cref{def:bijective}), with interventions transported accordingly.
The pair is a bijective translation, hence exact, and is induced by the cells $\mathbf V^{(b)}$ with identity cell maps, so it lies in $\mathcal A(L,H_{\mathrm{id}})$; because condition (a) ranges over every target $H$, it applies to this abstraction.

\smallskip
\noindent\emph{Step 2: blockwise factorization of $T^{-1}$.}
Condition (a) says $\tau_{\mathrm{id}}\circ T^{-1}$ is constructive under the fixed transported cells $\sigma(\mathbf V^{(b)})$, so it supplies surjective cell maps
\begin{equation}
\pi_b^T:\Val_{\sigma(\mathbf V^{(b)})}\twoheadrightarrow\Val_{\mathbf V^{(b)}}
\end{equation}
whose outputs are the complete original block values.
Equating the $b$-th component of the alignment-induced state map with the $b$-th component of $\tau_{\mathrm{id}}\circ T^{-1}$ gives
\begin{equation}
(T^{-1}(\mathbf y))_{\mathbf V^{(b)}}=\pi_b^T(\mathbf y_{\sigma(\mathbf V^{(b)})})
\end{equation}
for every $b$: each output block of $T^{-1}$ depends only on its assigned transported input block.

\smallskip
\noindent\emph{Step 3: the block factors are bijective.}
Each $\pi_b^T$ is surjective as a cell map in the transported alignment.
For injectivity, suppose $\mathbf u\neq\mathbf u'$ but $\pi_b^T(\mathbf u)=\pi_b^T(\mathbf u')$; completing $\mathbf u$ and $\mathbf u'$ to total states $\mathbf y\neq\mathbf y'$ that agree on every other transported cell, the component equations give $T^{-1}(\mathbf y)=T^{-1}(\mathbf y')$, contradicting injectivity of $T^{-1}$.
Hence every $\pi_b^T$ is bijective.

\smallskip
\noindent\emph{Step 4: assembling $D$.}
With the bijection $\sigma_b:\Val_{\mathbf V^{(b)}}\to\Val_{\sigma(\mathbf V^{(b)})}$, define
\begin{equation}
(D^{(b)})^{-1}:=\pi_b^T\circ\sigma_b:\Val_{\mathbf V^{(b)}}\to\Val_{\mathbf V^{(b)}},
\end{equation}
and let $D^{-1}:=\Diag((D^{(1)})^{-1},\ldots,(D^{(B)})^{-1})$.
Comparing $D^{-1}\circ\sigma^{-1}$ with $T^{-1}$ on each block:
\begin{align}
((D^{-1}\circ\sigma^{-1})(\mathbf y))_{\mathbf V^{(b)}}
&=(D^{(b)})^{-1}\!\left(\sigma_b^{-1}(\mathbf y_{\sigma(\mathbf V^{(b)})})\right)\\
&=\pi_b^T(\mathbf y_{\sigma(\mathbf V^{(b)})})\\
&=(T^{-1}(\mathbf y))_{\mathbf V^{(b)}},
\end{align}
by the definition of $(D^{(b)})^{-1}$ and the block-factorization equation of Step~2.
The maps agree on every block, so $T^{-1}=D^{-1}\circ\sigma^{-1}$ and $T=\sigma\circ D$, which is condition (b).
Exactness is not used in this algebraic factorization; it is needed in Step~1 to place the identity abstraction in the family quantified over by condition (a).
\end{proof}

\subsection{Uniqueness up to macro value relabeling}\label{app:proof-unique}

Two surjective cell maps with the same fibers identify the same low-level values and therefore differ only by a relabeling of their high-level values. The following corollary applies this observation to abstractions related by a block reparameterization.

\begin{cor}[Uniqueness up to macro value relabeling]\label{cor:unique}
Let $L_2=(L_1)_D$ for $D\in\mathcal D$.
Let $(\tau_1,\omega_1):L_1\to H_1$ and $(\tau_2,\omega_2):L_2\to H_2$ be constructive abstractions with the same cells $\mathbf V^{(b)}$ and cell maps $\pi_{1,b}$, $\pi_{2,b}$, and suppose that for every $b$ the fibers of $\pi_{2,b}$ and of $\pi_{1,b}\circ(D^{(b)})^{-1}$ coincide on $\Val_{\mathbf V^{(b)}}$.
Then there is a unique blockwise bijection $\phi=(\phi_b)_b$, $\phi_b:\Val_{Z^{(b)}_1}\to\Val_{Z^{(b)}_2}$, with $\tau_2=\phi\circ\tau_1\circ D^{-1}$ and $\omega_2(I^D)=\phi\circ\omega_1(I)\circ\phi^{-1}$ for every $I\in\Psi$, and $H_1$ and $H_2$ have $\phi$-related solution sets under every intervention in $\omega_1(\Psi)$.
If moreover $\pi_{2,b}=\pi_{1,b}\circ(D^{(b)})^{-1}$ for every $b$, then $\phi=\id$ and the two macro-models have the same solution sets under every such intervention.
\end{cor}

\begin{proof}[Proof of \cref{cor:unique}]
By \cref{thm:maximality} with $\sigma=\id$, $(\tau_1\circ D^{-1},\omega_1^D)$ is constructive $L_2\to H_1$ with cell maps $\pi_{1,b}\circ(D^{(b)})^{-1}$.
Fix $b$: this map and $\pi_{2,b}$ are surjections of $\Val_{\mathbf V^{(b)}}$ with the same fibers, so $\phi_b(z):=\pi_{2,b}(\mathbf u)$ for any $\mathbf u\in(\pi_{1,b}\circ(D^{(b)})^{-1})^{-1}(z)$ is a well-defined bijection, and it is the only map with $\pi_{2,b}=\phi_b\circ\pi_{1,b}\circ(D^{(b)})^{-1}$ because the latter map is onto.
Hence $\tau_2=\phi\circ\tau_1\circ D^{-1}$, and by the definition of the induced intervention map (\cref{def:constructive}) the intervention $I^D$ setting the cells $S$ to $D^{(b)}(\mathbf i_b)$ is sent by $\omega_2$ to $\doop(Z^{(b)}_2{=}\phi_b(\pi_{1,b}(\mathbf i_b)),\,b\in S)$, which is $\phi\circ\omega_1(I)\circ\phi^{-1}$ because $\phi$ acts coordinatewise.
Exactness of both pairs then gives, for every $I\in\Psi$, $\phi(\Solve(H_1^{\omega_1(I)}))=\phi(\tau_1(\Solve(L_1^I)))=\tau_2(\Solve(L_2^{I^D}))=\Solve(H_2^{\omega_2(I^D)})$, the empty intervention included; the last claim follows because equal cell maps force $\phi_b=\id$.
\end{proof}

\section{Concept-Directed Steering and Transport Details}\label{app:exp-additive}

\paragraph{Generator and model.}
Each example draws $s_1,u_1,s_2,u_2$ independently from $\mathrm{Uniform}[-1,1]$, lifts each pair by $\phi(s,u)=(s,u,s^2,su,u^2)^\top$, and forms
\begin{equation}
    \mathbf o=A_1\phi(s_1,u_1)+A_2\phi(s_2,u_2)\in\mathbb R^{10},
\end{equation}
where $A=[A_1\ A_2]\in\mathbb R^{10\times10}$ is the orthogonal factor of one seeded random matrix and is fixed across all splits and models.
The construction satisfies the reference-side sufficient-nonlinearity condition of \citet{lachapelle_additive_2023}; the exact learned-side conditions of that result are not claimed for finite trained networks.
For model $r$, the encoder $E_r$ has widths $10,64,64,4$, and each branch of the decoder $G_r$ has widths $2,64,64,10$, with $\tanh$ hidden activations and linear outputs.
Training minimizes reconstruction mean squared error (MSE) with Adam, using batch size $1024$, learning rate $10^{-3}$, and weight decay $5\times10^{-4}$.
Training runs for at most $2000$ epochs and stops after $200$ epochs without a validation-MSE improvement of at least $10^{-7}$; the best validation checkpoint is retained.
All runs use CPU-only local execution with four PyTorch threads.
No factor label enters autoencoder training or checkpoint selection.

\paragraph{Data splits and fitted maps.}
The experiment uses $20{,}000/2{,}000$ observations for autoencoder training and validation and $5{,}000$ held-out observations for reconstruction eligibility.
The reference cell maps use separate calibration, validation, and test splits of sizes $2{,}000$, $2{,}000$, and $5{,}000$.
The cross-model maps use $5{,}000/2{,}000$ paired, unlabelled observations for fitting and validation.
For each of $5{,}000$ further base observations, we independently sample one intervention value $s_1^\star$ and one intervention value $s_2^\star$ from $\mathrm{Uniform}[-1,1]$.
This gives $5{,}000$ steering cases per concept, or $10{,}000$ interventions in total; each corresponding edit is evaluated in the reference model and four retained recipient models.

For each possible pairing of reference block $b$ and concept $j$, an affine candidate map
\begin{equation}
    \pi_{bj}(\mathbf x_0^{(b)})=\mathbf w_{bj}^\top\mathbf x_0^{(b)}+c_{bj}
\end{equation}
is fitted by least squares. The one-to-one assignment with the highest validation $R^2$ is retained, and its two maps become $\pi_1$ and $\pi_2$.
For every possible reference--recipient block pair, an affine map is fitted in standardized coordinates by least squares.
The one-to-one pairing with the lowest worst-block validation NMSE estimates $\widehat\sigma_r$ and is retained.
After indexing each recipient block by its matched reference block, the selected within-block map $\widehat D_r^{(b)}$ satisfies $\widehat D_r^{(b)}(\mathbf x_0^{(b)})\approx\mathbf x_r^{(b)}$.
This second fit sees paired codes of the same observations but no factor values.

\paragraph{Selecting a point in the fiber.}
Let $\mathbf x_0^{(b)}$ be the intervened block of the reference code for the base observation.
For the affine $\pi_b$, the closest point in its fiber over $s_b^\star$ is
\begin{equation}
    \mathbf x_0^{(b)\star}
    =\mathbf x_0^{(b)}
     +\frac{s_b^\star-\pi_b(\mathbf x_0^{(b)})}{\lVert\mathbf w_b\rVert^2}\mathbf w_b,
    \qquad \pi_b(\mathbf x_0^{(b)\star})=s_b^\star.
    \label{eq:steering-preimage}
\end{equation}
Replacing this block produces the reference code $\mathbf x_0^\star$ and decoded observation $\widehat{\mathbf o}_0^\star=G_0(\mathbf x_0^\star)$.
Recipient model $r$ receives $\mathbf x_r^{(b)\star}=\widehat D_r^{(b)}(\mathbf x_0^{(b)\star})$, producing $\mathbf x_r^\star$ and $\widehat{\mathbf o}_r^\star=G_r(\mathbf x_r^\star)$.
This Euclidean projection is a convenient way to choose one representative of the fiber; it is not a unique inverse of $\pi_b$.

\paragraph{Evaluation.}
For each intervention, we construct one exact generator observation by replacing only the intervened concept value in the factors of the base observation:
\begin{equation}
    \mathbf o^{\mathrm{gen}}
    =A_b\phi(s_b^\star,u_b)
     +A_{3-b}\phi(s_{3-b},u_{3-b}),
\end{equation}
where $u_b$ and the second factor pair come from the base observation.
The same $\mathbf o^{\mathrm{gen}}$ is used to evaluate the complete reference and recipient observations by computing the NMSE between $\widehat{\mathbf o}_r^\star$ and $\mathbf o^{\mathrm{gen}}$ for $r=0$ and every retained recipient.
Because the generator is injective and $\phi$ contains the original factor coordinates, this common comparison tests whether the edit sets $s_b$ to $s_b^\star$ while preserving its nuisance value and the other factor pair.

For predictions $\widehat{\mathbf y}_1,\ldots,\widehat{\mathbf y}_N$ and comparisons $\mathbf y_1,\ldots,\mathbf y_N$, we define
\begin{equation}
    \operatorname{NMSE}(\widehat{\mathbf y},\mathbf y)
    :=\frac{\sum_{n=1}^N\lVert\widehat{\mathbf y}_n-\mathbf y_n\rVert^2}
    {\sum_{n=1}^N\lVert\mathbf y_n-\overline{\mathbf y}\rVert^2},
    \qquad
    \overline{\mathbf y}:=\frac{1}{N}\sum_{n=1}^N\mathbf y_n.
\end{equation}
For cross-model map validation, each recipient coordinate is standardized using the corresponding training-set mean and standard deviation before applying this definition.

The thresholds were fixed before the run: reconstruction NMSE at most $0.01$; reference-readout and decoded-concept $R^2$ at least $0.95$; cross-model map NMSE at most $0.05$; untargeted $s$ and $u$ $R^2$ at least $0.95$; generator-validity residual at most three times the reconstruction floor; and final common-comparison NMSE at most $0.02$.

\paragraph{Numerical results.}
The reference model reconstructs with NMSE $0.007312$.
Four of ten additional autoencoders pass the same gate, with NMSE between $0.007762$ and $0.007927$; the other six lie between $0.099071$ and $0.100609$.
The two selected reference cell maps have test $R^2$ between $0.999602$ and $0.999618$, while each block has no predictive power for the other concept.
The selected cross-model maps have validation NMSE between $0.001985$ and $0.005839$; the alternative block pairings have NMSE near $1$.

All retained models satisfy the reconstruction criterion, and every predeclared calibration, map-validation, steering, preservation, validity, and common-observation check passes.

\end{document}